%% file: main.tex
\documentclass[letterpaper, 10pt, conference]{ieeeconf}  

\IEEEoverridecommandlockouts                             
\makeatletter
\let\NAT@parse\undefined
\makeatother
\usepackage[pagebackref=true,breaklinks=true,colorlinks=true,bookmarks=false,allcolors=blue,citecolor=blue]{hyperref}
\makeatletter
\def\footnoterule{\relax%
  \kern-5pt
  \hbox to \columnwidth{\vrule width 0.5\columnwidth height 0.4pt\hfill}
  \kern4.6pt}
\makeatother

\usepackage{url}

\usepackage{graphics}

\usepackage{amsthm}
\usepackage{times} % assumes new font selection scheme installed
\usepackage{amsmath} % assumes amsmath package installed
\usepackage{amssymb}  % assumes amsmath package installed
\usepackage{bm}
\usepackage{enumerate}

\theoremstyle{plain}

\newtheorem{theorem}{Theorem}

\newtheorem{prop}{Proposition}
\newtheorem{corollary}{Corollary}

\newtheorem{assumption}{Assumption}
\DeclareMathOperator{\rank}{rank}

\usepackage[ruled,vlined]{algorithm2e}

\usepackage{color}

\usepackage{graphicx}

\title{\LARGE \bf
Active Depth Estimation: Stability Analysis and its Applications
}
\author{R\^{o}mulo T. Rodrigues$^{1}$, Pedro Miraldo$^{2}$, Dimos V. Dimarogonas$^{2}$, and A. Pedro Aguiar$^{1}$% <-this % stops a space
\thanks{This work was supported by PDMA-NORTE-08-5369-FSE-000061, UIDB/00147/2020 SYSTEC through the FCT/MCTES (PIDDAC); FCT projects POCI-01-0145-FEDER-031823; IMPROVE, POCI-01-0145-FEDER-031411-HARMONY; STRIDE NORTE-01-0145-FEDER-000033; LARSyS - FCT Plurianual funding 2020-2023; the Swedish research council (VR); Swedish Foundation for Strategic Research (SSF); and the Knut och Alice Wallenberg Foundation (KAW).}% <-this % stops a space
\thanks{$^{1}$R. T. Rodrigues and A. P. Aguiar are with the Research Center for Systems and Technologies (SYSTEC), Faculty of Engineering, University of Porto, Porto, Portugal.\newline
        E-Mail:{\tt\small rtr@fc.up.pt} and {\tt\small pedro.aguiar@fe.up.pt}.} %
\thanks{$^{2}$P. Miraldo and D. V. Dimarogonas are with the Division of Decision and Control Systems, KTH Royal Institute of Technology, Stockholm, Sweden.\newline
        E-Mail:{\tt\small \{miraldo,dimos\}@kth.se}.}%
}

\makeatletter
\renewcommand{\maketag@@@}[1]{\hbox{\m@th\normalsize\normalfont#1}}%
\makeatother

\begin{document}
\maketitle
\thispagestyle{empty}
\pagestyle{empty}

%%%%%%%%%%%%%%%%%%%%%%%%%%%%%%%%%%%%%%%%%%%%%%%%%%%%%%%%%%%%%%%%%%%%%%%%%%%%%%%%
\begin{abstract}
Recovering the 3D structure of the surrounding environment is an essential task in any vision-controlled Structure-from-Motion (SfM) scheme. This paper focuses on the theoretical properties of the SfM, known as the incremental active depth estimation. The term incremental stands for estimating the 3D structure of the scene over a chronological sequence of image frames. Active means that the camera actuation is such that it improves estimation performance. Starting from a known depth estimation filter, this paper presents the stability analysis of the filter in terms of the control inputs of the camera. By analyzing the convergence of the estimator using the Lyapunov theory, we relax the constraints on the projection of the 3D point in the image plane when compared to previous results. Nonetheless, our method is capable of dealing with the cameras' limited field-of-view constraints. The main results are validated through experiments with simulated data.
\end{abstract}

\input{files/intro.tex}
\input{files/background.tex}
\input{files/our1_Thrm.tex}
\input{files/our2_SP.tex}
\input{files/results.tex}
\input{files/conclusion.tex}

\bibliographystyle{IEEEtran}
\bibliography{files/sacrpe}

\end{document}

%% file: files/intro.tex
\section{INTRODUCTION}
Structure-from-Motion (SfM) aims at recovering the 3D structure of the environment from a moving camera. It is used when the motion of the camera and its intrinsic parameters are known. This is one of the more important modules in applications such as: autonomous navigation \cite{olson02}, UAV flight control \cite{li08}, robot hand-eye calibration \cite{andreff01}, topographic surveying \cite{clapuyt16}, and multi-robot relative pose estimation~\cite{romulo19}. The SfM problem has been studied for the last three decades by the roboticists and computer vision researchers. Below, we categorize available solutions as geometric/filtering based methods and passive/active techniques.

Geometric-based techniques \cite{koenderink91,bartoli05,schonberger16} often apply triangulation for estimating the depth of the points from two or more different viewpoints. The frames do not need to be consecutive, and this method is usually followed by an offline non-linear refinement such as bundle adjustment \cite{triggs00}. Geometric-based techniques provide accurate results but suffer from small baseline camera displacements. On the other hand, filter or incremental-based methods, such as \cite{civera08,luca08,martinelli12}, explicitly consider the dynamics of projected 3D points into a sequence of continuously acquired images. Incremental strategies focus on efficient computation and take advantage of the small continuous motions of the camera (small displacements). Besides, incremental-based techniques aim at getting a robust estimation of the model uncertainties. 

The works mentioned in the previous paragraph are passive, i.e., the camera motion is not used to the goal of mapping the 3D environment. In the last decade, some authors have been studying the use of active vision techniques to assist the structure-from-motion modules. The authors in \cite{chaumette06} propose the use of 3D reconstruction goals in the control loop. They use the proposed method in the reconstruction of 3D points, cylinders, straight lines, and spheres.
In \cite{spica13}, the authors address an active strategy for tuning the transient response of a particular class of nonlinear observers that are well suited for active SfM problems. The technique is applied to a 3D point active SfM scheme. The framework was later used for the cases of cylinder, spheres (see \cite{spica14}), 3D planes (in \cite{spica15}), and 3D straight lines (see \cite{mateus18,mateus19}). There are also works on high level controllers based on SfM. For example, \cite{romulo18} presents a method to actively ensure the presence of good features in a structure-from-motion module, and \cite{costante2018} proposes an optimal path planning framework that maximizes the visual information during navigation. 

In this paper, we study the stability analysis for an incremental active SfM using point features. The goal is to understand under what conditions it is possible to obtain an online estimation of the unknown depth of a point feature, from any initial condition. We resort to the knowledge of the motion of the camera and the 2D image plane coordinates of the projected 3D point. Our work builds on top of the incremental depth estimator addressed in \cite{spica13, spica14}, where some guarantees for its stability and how to maximize its convergence speed were studied. However, for a point feature, the asymptotic stability result in \cite{spica13, spica14} only holds if 1) the camera motion drives the projection of the point to the origin of the image frame, and 2) the depth (unknown parameter being estimated) is constant after a transient. As a consequence, some issues arise in practical applications. For example, in \cite{spica17a}, the results of \cite{spica13} are applied to the coupled depth estimation and visual servo control problem. The strategy strives to increase the convergence speed, but the convergence properties are not met. This results from the requirement of translating the projection of a point to the origin of the image frame, which conflicts with the visual servoing goal.

In our work, we take a step back to first analyze the camera actuation policies that provide asymptotic stability guarantees on the depth estimation of a single feature. In contrast to previous works with similar stability properties, we do not require the tracked feature to lie in (or visit) the origin of the image frame. Moreover, the unknown depth is not necessarily constant throughout the estimation process. 

The next section presents the notations and background work. Section~\ref{sec:convergence_estimator} presents the stability analysis of the active filter. Then,  Sec.~\ref{sec:single_point_application} discusses its use in a single 3D point mapping application. Simulation results are shown in Sec.~\ref{sec:exp}, and Sec.~\ref{sec:conclusions} concludes the paper.

%% file: files/background.tex
\section{PRELIMINARIES}
\label{sec:problem_definition}
This section presents notations and background work that support the remainder of this document. 

\subsection{Notation}
Scalars are written in lower case letters and column vectors typed in bold symbol lower case letters. A vector can be split into smaller pieces using the notation $\mathbf{v}_{(i:j)} := [v_i, v_{i+1}, \dots, v_j]^T$. Matrices are printed in upper case letter, as well as the coordinates of a 3D Point. 

\subsection{Background}
Consider a camera moving freely in space and let $\{C\}$ be the coordinate frame attached to the origin of the sensor. The camera observes a static 3D point described in $\{C\}$ as $\mathbf{p}:=[X,Y,Z]^T \in \mathbb{R}^3$. Let $\mathbf{s}:= [x,y]^T = [X/Z,Y/Z]^T \in \mathbb{R}^2$ be the projection of $\mathbf{p}$ into the camera's normalized image plane and consider the change of variable $\chi=1/Z$. Applying the new variables in the well-known optical flow equation \cite{hartley03} gives
\begin{equation}
    \begin{bmatrix}
        \dot{x} \\ \dot{y}  \\ \dot{\chi}
    \end{bmatrix} =
    \begin{bmatrix}
        -\chi & 0 & x\chi & xy & -(1+x^2) & y \\
        0 & -\chi & y\chi & 1+y^2 & -xy & -x \\
        0 & 0 & \chi^2 & y\chi & -x\chi & 0
    \end{bmatrix}
    \begin{bmatrix}
    \mathbf{v} \\ \mathbf{w}
    \end{bmatrix},
    \label{eq:point_velocity_camera_frame_normalized}
\end{equation}
where $\mathbf{v}:=[v_x,v_y,v_z]^T \in \mathbb{R}^3$ and $\mathbf{w}:=[w_x,w_y,w_z]^T \in \mathbb{R}^3$ are the camera linear and angular velocities described in $\{C\}$. The dynamics of the system can be stated in compact form
\begin{equation}
\begin{cases}
    \dot{\mathbf{s}} &= J_v\mathbf{v}\chi + J_w\mathbf{w} \\
    \dot{\chi} &= J_q\mathbf{v}\chi^2 + J_l\mathbf{w}\chi
    \label{eq:system_dynamics_I}
\end{cases},
\end{equation}
where
\begin{equation}
\begin{cases}
    J_v = \begin{bmatrix}
        -1 & 0 & x  \\
        0 & -1 & y 
    \end{bmatrix}\\
     J_w = \begin{bmatrix}
        xy & -(1+x^2) & y \\
        1+y^2 & -xy & -x 
    \end{bmatrix}\\
    J_q = \begin{bmatrix}
        0 & 0 & 1 
    \end{bmatrix}, \qquad
    J_l = \begin{bmatrix}
        y & -x & 0 
    \end{bmatrix}\\
\end{cases}.
\label{eq:system_matrices}
\end{equation}

Given $\mathbf{s}, \mathbf{v}\text{, and }\mathbf{w}$, we want to estimate the unknown depth described by $\chi$ (also denoted as unmeasurable variable). For that, consider the following notations. The estimation variables are $\hat{\mathbf{s}}$ and $\hat{\chi}$. The respective estimation errors are $\tilde{\mathbf{s}} = \mathbf{s} - \hat{\mathbf{s}}$ and $\tilde{\chi} = \chi - \hat{\chi}$. The state estimation problem addressed here uses an observer similar to \cite{spica13}:
\begin{equation}
    \begin{cases}
    \dot{\hat{\mathbf{s}}} = J_v \mathbf{v} \hat{\chi} + J_w\mathbf{w} + k_{s}\tilde{\mathbf{s}} \\
    \dot{\hat{\chi}} = J_q \mathbf{v} \hat{\chi}^2  + J_l \mathbf{w} \hat{\chi} + k_{\chi}(J_v \mathbf{v})^T \tilde{\mathbf{s}}
    \end{cases},
    \label{eq:observer_strategy}
\end{equation}
where $k_s,k_{\chi}\in\mathbb{R}^+$ are the control gains. The corresponding estimation error dynamics is
\begin{equation}
    \begin{cases}
    \dot{\tilde{\mathbf{s}}} = J_v \mathbf{v} \tilde{\chi} - k_s\tilde{\mathbf{s}} \\
    \dot{\tilde{\chi}} =\tilde{\chi}(J_q \mathbf{v}({\chi} + \hat{\chi})  + J_l \mathbf{w}) - k_{\chi}(J_v \mathbf{v})^T \tilde{\mathbf{s}}
    \end{cases}.
    \label{eq:observer_error_dynamics}
\end{equation}

%% file: files/our1_Thrm.tex
\section{CONVERGENCE OF THE ESTIMATOR}
\label{sec:convergence_estimator}
In this section we provide the stability analysis of the depth estimation filter. The goal is to provide guarantees for the convergence of the unmeasurable depth for recovering the 3D structure of the world given by $\mathbf{p} = [\mathbf{s},1]/\chi$.

\begin{assumption}
The observed 3D point cannot lie behind the camera. Consequently, we restrict our analysis to the domain where $\chi$ is positive, that is, we assume $\chi \geq 0, \forall t$.
\label{assumption:chi_positive}
\end{assumption}

This assumption has an explicit physical meaning. In fact, cameras are not able to observe 3D points that are behind them. This would require a negative depth.

\begin{theorem}
Consider the estimator \eqref{eq:observer_strategy} for the dynamic system \eqref{eq:system_dynamics_I} under Assumption~\ref{assumption:chi_positive}. The equilibrium point $(\tilde{\mathbf{s}}, \tilde{\chi}) = \mathbf{0}$ is stable and the estimation error converges to zero as $t \rightarrow \infty$ provided that $\forall t \geq t_0$ the following constraints hold simultaneously:
\begin{enumerate}
    \item $J_l\mathbf{w}\leq0$;
    \item $\begin{cases}
J_q\mathbf{v} \leq 0 \text{ , if } \hat{\chi} >  0 \\ 
J_q\mathbf{v} = 0 \text{ , otherwise}
\end{cases}$;
    \item $\sigma^2 = (xv_z - v_x)^2 + (yv_z - v_y)^2 > 0$;
\end{enumerate}
where $\mathbf{v}$, $\mathbf{w}$, and their time-derivatives are bounded signals.
\label{the:theorem_generic}
\end{theorem}

\begin{proof}
Consider the Lyapunov function candidate%\cite{khalil02}
\begin{equation}\label{eq:used_lyapunov}
    V(\tilde{\mathbf{s}}, \tilde{\chi}) = \frac{1}{2}\|\tilde{\mathbf{s}}\|^2 + \frac{1}{2k_{\chi}}\tilde{\chi}^2,
\end{equation}
with $k_{\chi} > 0$, and its time-derivative
\begin{align}
    \dot{V} &= \tilde{\mathbf{s}}^T\dot{\tilde{\mathbf{s}}} + \frac{1}{k_{\chi}}\tilde{\chi}\dot{\tilde{\chi}}
    \label{eq:used_lyapunov_timederivative}
\end{align}

Substituting \eqref{eq:observer_error_dynamics} in the previous equation:
\begin{align}
    \dot{V} &= \tilde{\mathbf{s}}^T(J_v\mathbf{v}\tilde{\chi} - k_{s}\tilde{\mathbf{s}}) + \nonumber \\
    & \qquad + \frac{1}{k_{\chi}}\tilde{\chi}(J_q\mathbf{v}(\chi+\hat{\chi})\tilde{\chi}+J_l\mathbf{w}\tilde{\chi} -  k_{\chi}(J_v\mathbf{v})^T\tilde{s})\\
    &= -\tilde{\mathbf{s}}^T k_{s} \tilde{\mathbf{s}} + \frac{1}{k_{\chi}}\tilde{\chi}J_q\mathbf{v}(\chi+\hat{\chi})\tilde{\chi} + \frac{1}{k_{\chi}}\tilde{\chi}J_l\mathbf{w}\tilde{\chi}.
    \label{eq:final_first_the_cons}
\end{align}
By combining Assumption~\ref{assumption:chi_positive} and the input constraints stated in Theorem~\ref{the:theorem_generic}, we have that the three terms in the right-hand side of \eqref{eq:final_first_the_cons} are non-positive. Hence, $\dot{V} \leq 0$ and the equilibrium point $(\tilde{\mathbf{s}}, \tilde{\chi})=\mathbf{0}$ is stable. We also conclude that $V(t)\leq V(t_0)$, and therefore, that the signals $\tilde{\mathbf{s}}$ and $\tilde{\chi}$ are bounded.

The critical case that precludes asserting asymptotically stability from \eqref{eq:final_first_the_cons} occurs when $J_q\mathbf{v}=0$ and $J_l\mathbf{w}=0$, and consequentially, $\dot{V}=-\tilde{\mathbf{s}}^T k_{s} \tilde{\mathbf{s}}$. Let $N(\cdot)$ denote the nullspace of a matrix, then $J_l\mathbf{w}=J_q\mathbf{v}=0$ either because the feature lies in the origin of the image plane ($\mathbf{s} = [0,0]^T$), or because $\mathbf{v} \in N(J_q)$ and $\mathbf{w} \in N(J_l)$ simultaneously.
For $J_q\mathbf{v}=0$ and $J_l\mathbf{w}=0$, the second derivative of the Lyapunov candidate function is
\begin{equation}
    \ddot{V} = -2\tilde{\mathbf{s}}^Tk_s\dot{\tilde{\mathbf{s}}}= -2\tilde{\mathbf{s}}^Tk_s(J_v\mathbf{v}\tilde{\chi} - k_{s}\tilde{\mathbf{s}}).
\end{equation}

As $\tilde{\mathbf{s}}$, $\tilde{\chi}$, and $\mathbf{v}$ (by definition) are bounded, the function $\ddot{V}$ is also bounded. Thus, $\dot{V}$ is uniformly continuous and from Barbalat's Lemma \cite{slotine05}, we have that ${\tilde{\mathbf{s}}} \rightarrow 0$ as $t\rightarrow \infty$. Now, for the asymptotic behaviour of $\tilde{\chi}$ when $J_q\mathbf{v}=0$ and $J_l\mathbf{w}=0$, from  \eqref{eq:observer_error_dynamics} we have, as $t \rightarrow \infty$,
\begin{equation}
    \begin{cases}
    \underset{t \to \infty}{\lim}\dot{\tilde{\mathbf{s}}} =  \underset{t \to \infty}{\lim}J_v\mathbf{v}\tilde{\chi} \\
    \underset{t \to \infty}{\lim}\dot{\tilde{\chi}} =  0
    \end{cases}.
\end{equation}
The second equation states that the depth estimation error becomes a constant, but not necessarily zero. To show that indeed it will converge to zero, we first show that $\dot{\tilde{\mathbf{s}}}$ is uniformly bounded because its time derivative given by
\begin{align}
\ddot{\tilde{\mathbf{s}}} &= \dot{J}_v \mathbf{v}\tilde{\chi} + J_v \dot{\mathbf{v}}\tilde{\chi} + J_v \mathbf{v}\dot{\tilde{\chi}} - k_s\dot{\tilde{\mathbf{s}}} \\
&= \dot{J}_v \mathbf{v}\tilde{\chi} + J_v \dot{\mathbf{v}}\tilde{\chi} - k_{\chi}J_v \mathbf{v}(J_v \mathbf{v})^T\tilde{\mathbf{s}} - k_s J_v \mathbf{v}\tilde{\chi} + k_s^2 \tilde{\mathbf{s}}
\end{align}
is a function of bounded signals. Thus, since $\tilde{\mathbf{s}}$ converges to the origin and $\dot{\tilde{\mathbf{s}}}$ is uniformly bounded, we conclude that $\dot{\tilde{\mathbf{s}}} \rightarrow 0$ as $t \rightarrow \infty$.
Consequently, we have that
\begin{equation}
    \underset{t \to \infty}{\lim}\dot{\tilde{\mathbf{s}}} = \underset{t \to \infty}{\lim}J_v\mathbf{v} \underset{t \to \infty}{\lim}\tilde{\chi} = 0.
\end{equation}

It must be the case that either $\tilde{\chi} \rightarrow 0$ or $J_v\mathbf{v} \rightarrow 0$. If the function $J_v\mathbf{v}$ is persistently exciting through all time, then the depth estimation error converges to zero. The signal $J_v\mathbf{v}$ is persistently exciting if the integral
\begin{equation}
    \int_{t_0}^t(J_v\mathbf{v})^TJ_v\mathbf{v}d\tau
\end{equation}
is positive definite $\forall t \geq t_0$.
Hence, the persistency of excitation (PE) condition holds if 
\begin{equation}
\sigma^2 = (J_v\mathbf{v})^TJ_v\mathbf{v} > 0, 
\label{eq:PE}
\end{equation}
which is the case from condition (3) in Theorem 1.

Thus, one can now conclude that the equilibrium point $(\tilde{\mathbf{s}}^T, \tilde{\chi}) = \mathbf{0}$ is asymptotically stable.
\end{proof}

%% file: files/our2_SP.tex
% \section{APPLICATION}
\section{CONSTRAINED ACTIVE DEPTH ESTIMATION}
\label{sec:single_point_application}
Any vision-based control scheme has to consider an important limitation of image sensors, its limited field of view. While tracking the projected 3D point (related to the unknown depth to be estimated), one needs to make sure the projection does not leave the image space. To achieve that, we have to include constraints on the motion of the camera. This section explores the theoretical stability guarantees derived in Sec.~\ref{sec:convergence_estimator} for active depth estimation, while ensuring the tracked projected point does not leave the image space.

To address the constraints on the camera motion, we introduce the continuous and smooth desired signal $\mathbf{s}_{des}(t)$ and define the tracking error
\begin{equation}
\mathbf{e}(t) = \mathbf{s}(t)-\mathbf{s}_{des}(t).
\end{equation}

The signal $\mathbf{s}_{des}$ is chosen such that the feature remains within the field of view of the camera during the depth estimation process. Assume that the feedback control law $\bm{\pi}(t, \mathbf{s}, \mathbf{s}_{des})$ drives the tracking error to the origin\footnote{For instance, if $\mathbf{s}_{des}$ is constant, then the proportional controller $\bm{\pi} = -k_p(\mathbf{s}-\mathbf{s}_{des})$, where $k_p \in \mathbb{R}^+$ ensures the desired behaviour.}, i.e., $\dot{\mathbf{s}} = \bm{\pi}, \forall t\geq t_0 \implies \mathbf{e}\rightarrow0$ as $t \rightarrow \infty$. From inspection of \eqref{eq:system_dynamics_I}, in addition to the camera's linear and angular velocities, $\dot{\mathbf{s}}$ depends on the unknown depth $\chi$. Thus, it is only possible to shape the dynamics of $\dot{\mathbf{s}}$ up to an estimation error. That being said, the goal is to design a control law for $(\mathbf{v},\mathbf{w})$ such that $\dot{\mathbf{s}}(\hat{\chi}, \mathbf{v},\mathbf{w})$ tracks the signal $\bm{\pi}(t, \mathbf{s}, \mathbf{s}_{des})$,
while:
\begin{enumerate}[(i)]
    \item imposing the constraints stated in Theorem~\ref{the:theorem_generic}, to assure that the stability property holds;
    \item improving the performance of the estimator, by maximizing $\sigma^2$ as defined in \eqref{eq:PE}; and
    \item accounting for the kinodynamics constraints of the camera described by $\|\mathbf{v}\| \leq v_{\max}$ and $\|\mathbf{w}\| \leq w_{\max}$, where $v_{\max}$ and $w_{\max}$ are the maximum linear and angular speed of the camera, respectively.
\end{enumerate}

Since constraints are most commonly not addressed when designing a control law to track the reference signal $\mathbf{s}_{des}$, simultaneously tracking $\bm{\pi}$ and respecting all the forementioned constraints can lead to an infeasible problem. A workaround is proposed by introducing a scale factor $\lambda_{\pi} \in [0,1]$ such that $\dot{\mathbf{s}}(\hat{\chi}, \mathbf{v},\mathbf{w})$ is required to track the reference $\lambda_{\pi}\bm{\pi}$. As the depth converges, tracking the scaled vector $\lambda_{\pi}\bm{\pi}$ -- rather than minimizing a norm error -- ensures that the path of the feature in the image frame follows the assignment specified by $\bm{\pi}$. This allows us to design a path for the feature that does not visit the origin of the image frame. The problem is formulated next:
\begin{equation}
\begin{aligned}
& \underset{\mathbf{v},\mathbf{w}, \lambda_{\pi}}{\text{maximize}}
& &  \lambda_\pi\\
& \text{subject to}
& & J_v\hat{\chi}\mathbf{v} + J_w\mathbf{w} = \lambda_{\pi}\bm{\pi} \\
&&& 0 \leq \lambda_{\pi} \leq 1  \\
&&& \text{constraints  (i), (ii), and (iii)}
\end{aligned}.
\label{problem:vs_1_generic}
\end{equation}
This problem is addressed in two configurations. The estimation strategy proposed in Section~\ref{subsec:non_constant_depth} does not implicitly impose the unknown depth to be constant. In contrast, Sec.~\ref{subsec:constant_depth} addresses the particular case that requires null depth rate. Both cases take advantage of the following Theorem:
\begin{theorem}
\label{thrm:optimization_generic}
Consider the non-convex problem:
\begin{equation}
\begin{aligned}
& \underset{\lambda_1, \lambda_2, \mathbf{v}_r}{\text{maximize}}
& &  \lambda_1\\
& \text{subject to}
& & \lambda_1 \mathbf{v}_1 + \lambda_2 \mathbf{v}_2 = r \mathbf{v}_r   \\
&&& \|\mathbf{v}_r\| = 1  \\
&&& 0 \leq\lambda_{1} \leq 1  \\
&&& -b \leq\lambda_{2} \leq b 
\end{aligned},
\label{eq:optimization_problem_generic}
\end{equation}
where $r,b \in \mathbb{R}^+, \mathbf{v}_1, \mathbf{v}_2 \in \mathbb{R}^n$, $\|\mathbf{v}_1\| > 0$, and $\|\mathbf{v}_2\| = 1$. 
The problem is always feasible if $r \leq b$.
\end{theorem}

Due to the lack of space, the reader is referred to \cite{rodrigues19r} for the proof of Theorem~\ref{thrm:optimization_generic} and a closed form solution for the problem in~\eqref{eq:optimization_problem_generic}, which is employed here. The solution does not impose restrictions on the feature coordinates, except the origin of the image frame, i.e., $\mathbf{s}=[0,0]^T$, which is a singularity. 

\subsection{Case: $\mathbf{s} \neq \mathbf{0}, \forall t \geq t_0$}
\label{subsec:non_constant_depth}
In this first scenario, $J_q\mathbf{v}=0$ and $J_w\mathbf{w}\leq0$. This allows us to to take advantage of Theorem~\ref{thrm:optimization_generic}, while still respecting the requirements for asymptotic convergence stated in Theorem~\ref{the:theorem_generic}. The PE condition of \eqref{eq:PE} simplifies to $\sigma^2 = v_x^2 + v_y^2 = \|\mathbf{v}_{(1:2)}\|^2$ and its maximum attainable value is limited by the kinodynamic constraint of the camera, $\sigma^2_{\max} = v_{\max}$. Under this scenario, the problem in~\eqref{problem:vs_1_generic} can be formulated as
% OPTIMIZATION PROBLEM
\begin{equation}
\begin{aligned}
& \underset{\mathbf{v},\mathbf{w}, \lambda_{\pi}}{\text{maximize}}
& &  \lambda_\pi\\
& \text{subject to}
& & \dot{\mathbf{s}}(\hat{\chi}, \mathbf{v}, \mathbf{w}) = \lambda_{\pi}\bm{\pi} \\
&&& 0 < \lambda_{\pi} \leq 1 \\
&&& J_q\mathbf{v} = 0\text{, } J_l\mathbf{w} \leq 0 \\
&&&  \|\mathbf{v}\| = v_{\max}\text{, }\|w\| \leq w_{\max}
\label{eq:optimization_problem_II}
\end{aligned},
\end{equation} 
and solved with the following proposition:
\begin{prop}
\label{prop:case_A}
Let the camera control input be 
\begin{equation}
\mathbf{v} = v_{\max} \begin{bmatrix} 
\mathbf{v}_r \\ 0
\end{bmatrix}
\text{ and }
\mathbf{w} = 
\begin{bmatrix} 
S \bm{\lambda}_s / \|\mathbf{s}\| \\ 0
\end{bmatrix},
\label{eq:input_substitution_II}
\end{equation}
and $S$, $J_{\bar{w}}$, and $\bm{\lambda}_{s}$ be defined as follows:
\begin{equation}
\begin{cases}
     S = \begin{bmatrix}
-\frac{\mathbf{s}_\perp}{\|\mathbf{s}_\perp\|} & \frac{\mathbf{s}}{\|\mathbf{s}\|}
\end{bmatrix} \\
     J_{\bar{w}} = \begin{bmatrix}
        xy & -(1+x^2) \\
        1+y^2 & -xy 
    \end{bmatrix} \\
    \bm{\lambda}_{s}= \begin{bmatrix} \lambda_{s_{\perp}} & \lambda_{s}\end{bmatrix}^T
\end{cases},
\label{eq:input_substituiton_matrices}
\end{equation}
where $\lambda_{s_{\perp}} \in \mathbb{R}^+$, $\lambda_{s}\in \mathbb{R}$, and $\mathbf{s}_\perp =[-y,x]^T$ is a vector perpendicular to $\mathbf{s}$. In particular, define $\bm{\lambda}_s$ as
\begin{equation}
\small
\bm{\lambda_{s}}=
    \begin{cases}
    \lambda_w\|\mathbf{s}\|  (J_{\bar{w}}S)^{-1}\bm{\pi} / \|\bm{\pi}\|\text{, if } (\|\bm{\pi}\| - \hat{\chi}v_{\max})\mathbf{s}^T\bm{\pi} < 0  \\
    \lambda_w\|\mathbf{s}\| [0, 1]^T\text{, otherwise} 
    \end{cases}.
\end{equation}

A sub-optimal solution for the problem in~\eqref{eq:optimization_problem_II} can be obtained by casting it in the shape of the problem in~\eqref{eq:optimization_problem_generic}, where the input variables are written as 
\begin{equation}
    \begin{cases}
    \mathbf{v}_1 = \bm{-\pi}\\
    \mathbf{v}_2 =
    \begin{cases}
    \bm{\pi}/\|\bm{\pi}\|\text{, if } (\|\bm{\pi}\| - \hat{\chi}v_{\max})\mathbf{s}^T\bm{\pi} < 0 \\
    \mathbf{s}_{\perp} / \|\mathbf{s}_{\perp}\|\text{, otherwise}
    \end{cases}\\
    r = \hat{\chi}v_{\max}\text{, } b = w_{\max}
    \end{cases},
\end{equation}
and the outputs mapped into
\begin{equation}
    \begin{cases}
    \lambda_{\pi} = \lambda_1^* \text{, } \lambda_{w} = \lambda_2^*; \\
    \mathbf{v}_r = \mathbf{v}_r^* 
    \end{cases}.
\end{equation}

\end{prop}

\begin{proof} 
First, we show that the control inputs are described as in \eqref{eq:input_substitution_II}. The constraint $J_q\mathbf{v} = 0$ implies that $v_z = 0$. Combining with $\|\mathbf{v}\| = v_{\max}$, the linear velocity vector can be written as $\mathbf{v} = v_{\max}[\mathbf{v}_r^T,  0]^T$, where ${\mathbf{v}_r} \in \mathbb{R}^2$ is a unit vector. For the angular velocity, re-write the constraint $ J_l\mathbf{w} \leq 0 $ using the slack variable $\lambda_{s_{\perp}}$, such that
\begin{equation}
    J_l\mathbf{w} \leq 0 \implies 
    \begin{cases}
    J_l\mathbf{w} &= \lambda_{s_{\perp}} \\
    \lambda_{s_{\perp}} &\leq 0
    \end{cases}.
\end{equation}
From $J_l\mathbf{w} = \lambda_{s_{\perp}}$ one concludes that $w_y = (y/x)w_x - (1/x)\lambda_{s_{\perp}}$. Applying this result into $J_w\mathbf{w}$: 
    \begin{align}
    J_w\mathbf{w} &= 
    J_w\begin{bmatrix}
    w_x \\ (y/x)w_x - (1/x)\lambda_{s_{\perp}}\\w_z
    \end{bmatrix} \\
    &=
    \begin{bmatrix}
    -y/x & y \\
    1 & - x
    \end{bmatrix}
    \begin{bmatrix}
    w_x \\ w_z
    \end{bmatrix} +
    \begin{bmatrix}
    (1/x+x) \\ y
    \end{bmatrix} \lambda_{s_{\perp}}.
    \end{align}
The column space of the first matrix on the right hand side of the previous equation has dimension 1 and, consequentially, it can be generated assuming $w_z =0$. Thus, the following equivalence holds:
\begin{equation}
    J_l\mathbf{w} = \lambda_{s_\perp} \implies -\mathbf{s}_{\perp}^T\mathbf{w}_{(1:2)} = \lambda_{s_\perp},
\end{equation}
where $\mathbf{s}_{\perp} = [-y, x]^T$. For $w_z=0$, we conclude that any feasible angular velocity can be described as
\begin{align}
\mathbf{w}_{(1:2)}&=-\frac{\mathbf{s}_\perp}{\|\mathbf{s}_\perp\| ^2}\lambda_{s_{\perp}} + \frac{\mathbf{s}}{\|\mathbf{s}\|^2}\lambda_{s}\\
&=\frac{1}{\|\mathbf{s}\|}\begin{bmatrix}
-\frac{\mathbf{s}_\perp}{\|\mathbf{s}_\perp\|} & \frac{\mathbf{s}}{\|\mathbf{s}\|}
\end{bmatrix}
\begin{bmatrix}
\lambda_{s_{\perp}} \\ \lambda_{s}
\end{bmatrix} \\
&=\frac{1}{\|\mathbf{s}\|}S \bm{\lambda}_s,
\end{align} 
where $S$, $\bm{\lambda}_s$, and $\lambda_{s}$ are as defined in \eqref{eq:input_substituiton_matrices}.
%This proves $\mathbf{w}$ to be as in \eqref{eq:input_substitution_II} and the proof for $\mathbf{v}$ is similar to the one presented  
 Within this setup the kinodynamics constraint $\|\mathbf{w}\| \leq w_{\max}$ is equivalent to $\|\bm{\lambda}_s\|  \leq \|\mathbf{s}\| w_{\max}$:
\begin{align}
   \|\mathbf{w}\| = \|\mathbf{w}_{(1:2)}\| &= \frac{1}{\|\mathbf{s}\|}\sqrt{ \bm{\lambda}_s^TS^TS\bm{\lambda}_s} \\
    &= \frac{\|\bm{\lambda}_s\|}{\|\mathbf{s}\|} \leq w_{\max}.
    \label{eq:alternative_kinodynamic_constraint}
\end{align}
This concludes the proof of \eqref{eq:input_substitution_II} and \eqref{eq:input_substituiton_matrices}.

Applying the control inputs into the first constraint of \eqref{eq:optimization_problem_II} and re-organizing the terms yields:
\begin{equation}
     \lambda_\pi(-\bm{\pi}) + J_w 
     \begin{bmatrix}
     S \bm{\lambda}_s / \|\mathbf{s}\| \\
     0
     \end{bmatrix}  = 
     -\hat{\chi}v_{\max}J_v
     \begin{bmatrix}
     {\mathbf{v}_r} \\
     0
     \end{bmatrix}
\end{equation}
\begin{equation}
     \lambda_{\pi} (-\bm{\pi})+\frac{1}{\|\mathbf{s}\|}J_{\bar{w}}S\bm{\lambda}_s =\hat{\chi}v_{\max}{\mathbf{v}_r}.
     \label{eq:triangle_equality_III}
\end{equation}
Let $\bm{\nu} = (1/\|\mathbf{s}\|)J_{\bar{w}}S\bm{\lambda}_s$ and notice that if $\|\bm{\pi}\| > \hat{\chi}v_{\max}$, $\bm{\lambda}_s$ must be such that $\bm{\pi}^T\bm{\nu} > 0$. On the contrary, if $\|\bm{\pi}\| < \hat{\chi}v_{\max}$, then one has to ensure $(-\bm{\pi})^T\bm{\nu} > 0$. Maximizing the dot product in both cases requires that $\bm{\nu}$ and $\bm{\pi}$ to be parallel. Both vectors are aligned if 
\begin{equation}
\bm{\lambda}_s \propto
(J_{\bar{w}}S)^{-1}\bm{\pi},
\label{eq:candidate_lambda_s}
\end{equation}
where the symbol $\propto$ denotes the relationship holds up to a scale factor. The matrix $S$ is orthogonal and, therefore, full rank. The matrix $J_{\bar{w}}$ is also full rank since $\det(J_{\bar{w}}) = 1 + x^2 + y^2 \neq 0$. From the Sylvester rank inequality, we have
\begin{equation}
    \rank(S) + \rank(J_{\bar{w}}) - 2 \leq \rank(J_{\bar{w}}S).
\end{equation}
Since both $S$ and $J_{\bar{w}}$ are $2\times2$ full rank  matrices, one concludes that their product is also full rank (and invertible). 

For feasibility, the first component of $\bm{\lambda}_s$ -- corresponding to $\lambda_{s_{\perp}}$ -- must be non-positive. Solving the right hand side of \eqref{eq:candidate_lambda_s}, $\lambda_{s_{\perp}}$ can be described  as 
\begin{equation}
    \lambda_{s,\perp} \propto
    \begin{cases}
    \mathbf{s} ^T\bm{\pi}\text{, if } \|\bm{\pi}\| > \hat{\chi}v_{\max} \\
    -\mathbf{s} ^T\bm{\pi}\text{, if } \|\bm{\pi}\| \leq \hat{\chi}v_{\max}
    \end{cases} .
    \label{eq:result_s_perp}
\end{equation}
If $\lambda_{s,\perp}$ is positive in either cases, it means that $\lambda_{s_{\perp}}=0$ is the largest feasible value that maximizes the projection of $\bm{\nu}$ into $\bm{\pi}$ or $(-\bm{\pi})$. Using a compact notation:
\begin{equation}
    \bm{\lambda}_{s}=
    \begin{cases}
    \lambda_w\|\mathbf{s}\|  (J_{\bar{w}}S)^{-1}\frac{\bm{\pi}}{\|\bm{\pi}\|}\text{, if } (\|\bm{\pi}\| - \hat{\chi}v_{\max})\mathbf{s}^T\bm{\pi} < 0  \\
    \lambda_w\|\mathbf{s}\|[0, 1]^T\text{, otherwise} 
    \label{eq:suboptimal_lambda_s}
    \end{cases},
\end{equation}
where $\lambda_w \in \mathbb{R}$. For the maximum feasible value of $\lambda_w$, compute the norm of the previous equation and compare with \eqref{eq:alternative_kinodynamic_constraint}. When $(\|\bm{\pi}\| - \hat{\chi}v_{\max})\mathbf{s}^T\bm{\pi} > 0$, we have
\begin{equation}
    \|\bm{\lambda}_s\| = \frac{\|\lambda_w\|\|\mathbf{s}\|}{\|\bm{\pi}\|}   \|(J_{\bar{w}}S)^{-1}\bm{\pi}\| \leq \|\mathbf{s}\|w_{\max}. 
\end{equation}
The singular values of $(J_{\bar{w}}S)^{-1}$ are 1 and $1/(1+x^2+y^2)$. Since the maximum singular value is 1, the upper bound $\|(J_{\bar{w}}S)^{-1}\bm{\pi}\| \leq \|\bm{\pi}\|$ holds and
\begin{align}
    \|\bm{\lambda}_s\| \leq \|\lambda_w\|\|\mathbf{s}\| &\leq \|\mathbf{s}\|w_{\max}, \\
    \|\lambda_w\| &\leq w_{\max}.
\end{align}

The same bound is obtained when $\bm{\lambda}_{s} = \lambda_w\|\mathbf{s}\|\begin{bmatrix}0 & 1\end{bmatrix}^T$ in \eqref{eq:suboptimal_lambda_s}:
\begin{equation}
    \|\lambda_w\|\mathbf{s}\|\begin{bmatrix} 0 & 1 \end{bmatrix}^T\| \leq \|\mathbf{s}\|w_{\max} \Rightarrow \|\lambda_w\| \leq w_{\max}.
\end{equation}

Finally, substituting \eqref{eq:suboptimal_lambda_s} in \eqref{eq:triangle_equality_III}:
\begin{equation}
\small
\hat{\chi}v_{\max}{\mathbf{v}_r} =
\begin{cases}
     \lambda_{\pi} (-\bm{\pi})+\lambda_w\frac{\bm{\pi}}{\|\bm{\pi}\|}\text{, if } (\|\bm{\pi}\| - \hat{\chi}v_{\max})\mathbf{s}^T\bm{\pi} > 0 \\
     \lambda_{\pi} (-\bm{\pi})+\lambda_w\frac{\mathbf{s}_{\perp}}{\|\mathbf{s}_{\perp}\|}\text{, otherwise}   
\end{cases},
\end{equation}
which allows us to obtain a sub-optimal solution for the problem in~\eqref{eq:optimization_problem_I} in the shape of the problem in~\eqref{eq:optimization_problem_generic} using the substitutions described by \eqref{eq:input_sfm} and \eqref{eq:output_sfm}.
\end{proof}

The sub-optimality comes from the fact that the solution consists in projecting $\bm{\lambda}_s$ into $\bm{\pi}$ when $(\|\bm{\pi}\| - \hat{\chi}v_{\max})\mathbf{s}^T\bm{\pi} < 0$. The projection is done via the mapping $J_{\bar{w}}S$. The singular values of $J_{\bar{w}}S$ are $1$ and $1+x^2+y^2$. Therefore, if $\mathbf{s}\neq[0,0]^T$, there can exist a $\bm{\lambda}_s$ that is not projected into $\bm{\pi}$, but the shear transformation performed by $J_{\bar{w}}S$ allows for a higher value of $\lambda_{\pi}$. Since in practical applications $1+x^2+y^2 \approx 1$, the solution obtained is not far from the optimal solution. The main advantage in our approach is that it is possible to compute a direction for $\bm{\lambda}_s$ in a closed-form.
\subsection{Case: $\mathbf{s} \neq \mathbf{0}$ and $\dot{\chi}=0, \forall t \geq t_0$}
\label{subsec:constant_depth}
Now, consider the specific scenario where the depth must be kept constant throughout the entire estimation process. For an unknown $\chi$ in  \eqref{eq:system_dynamics_I}, setting $J_q\mathbf{v}=0$ and $J_l\mathbf{w}=0$ guarantees that $\dot{\chi}=0$. Both aforementioned constraints are in accordance with Theorem~\ref{the:theorem_generic}. The problem, which is stated next:
% OPTIMIZATION PROBLEM
\begin{equation}
\begin{aligned}
& \underset{\mathbf{v},\mathbf{w}, \lambda_{\pi}}{\text{maximize}}
& &  \lambda_\pi\\
& \text{subject to}
& & \dot{\mathbf{s}}(\hat{\chi}, \mathbf{v}, \mathbf{w}) = \lambda_{\pi}\bm{\pi} \\
&&& 0 \leq \lambda_{\pi} \leq 1 \\
&&& J_q\mathbf{v} = 0\text{, } J_l\mathbf{w} = 0 \\
&&&  \|\mathbf{v}\| = v_{\max}\text{, }\|w\| \leq w_{\max}
\label{eq:optimization_problem_I}
\end{aligned},
\end{equation} 
is a particular case of problem~\eqref{eq:optimization_problem_II}. According to the following corollary, an optimal solution can be obtained using Theorem~\ref{thrm:optimization_generic}.
\begin{corollary}
Let the camera control input be described as
\begin{equation}
\mathbf{v} = v_{\max} \begin{bmatrix} 
{\mathbf{v}_r} \\ 0
\end{bmatrix}
\text{ and }
\mathbf{w} = \lambda_w \begin{bmatrix} 
\mathbf{s} / \|\mathbf{s}\| \\ 0
\end{bmatrix}.
\label{eq:input_substitution_I}
\end{equation}
Then, the problem in~\eqref{eq:optimization_problem_I} is equivalent to the problem in~\eqref{eq:optimization_problem_generic}, where 
\begin{equation}
    \begin{cases}
    \mathbf{v}_1 = \bm{-\pi}\text{, } \mathbf{v}_2 = \mathbf{s}_{\perp} / \|\mathbf{s}_{\perp}\| \\
    r = \hat{\chi}v_{\max}\text{, } b = w_{\max}
    \end{cases};
    \label{eq:input_sfm}
\end{equation}
and the outputs are mapped as:
\begin{equation}
    \begin{cases}
    \lambda_{\pi} = \lambda_1^* \text{, } \lambda_{w} = \lambda_2^* \\
    \mathbf{v}_r = \mathbf{v}_r^*
    \end{cases} .
    \label{eq:output_sfm}
\end{equation}
\end{corollary}

The proof is similar to the one presented in Sec.~\ref{subsec:non_constant_depth} by imposing $\lambda_{s_{\perp}} = 0$, that is, no slackness. In this case, the solution is optimal because the shear mapping is not involved.

%% file: files/results.tex
\section{EXPERIMENTS}\label{sec:exp}

The theoretical results derived in this work are validated using a numerical simulator. The following fixed parameters were employed: $v_{\max} = 0.1$~m/s, $w_{\max} = 0.15$~rad/s, $k_s=10$, and $k_{\chi}=2500$. The sampling time of the simulations is $0.05$~ms. In  Fig.~\ref{fig:results_comparison_01}, we compare the methods proposed in Sec.~\ref{subsec:non_constant_depth} and Sec.~\ref{subsec:constant_depth} with the one presented in  \cite{spica13,spica14}. For asymptotic stability, the strategy described in \cite{spica13,spica14} (continuous red line) and denoted here as \textit{Spica et al. (2014)}, requires the projection of the tracked 3D point to lie in the origin of the image plane and its corresponding depth to be constant, i.e., $\mathbf{s}_{des} = [0,0]^T$ and $\dot{\chi}=0$. The method presented in Sec.~\ref{subsec:non_constant_depth} (dashed green line) relaxes both requirements. The strategy described in Sec.~\ref{subsec:constant_depth} (continuous blue line) is a particular case of the previous method which keeps the unknown depth constant throughout the trajectory of the camera. Aiming at a fair comparison, the initial visual servoing error and the inverse depth estimation error are the same in the three cases. The initial configurations are: $\|\mathbf{e}(t_0)\| = 0.2~\text{m}$ and $\tilde{\chi}(t_0) = 0.9~\text{m}^{-1}$ (with ${\chi}(t_0) = 1~\text{m}^{-1}$ and $\hat{\chi}(t_0) = 0.1~\text{m}^{-1}$).

\begin{figure}
    \centering
        \includegraphics[width=.495\textwidth]{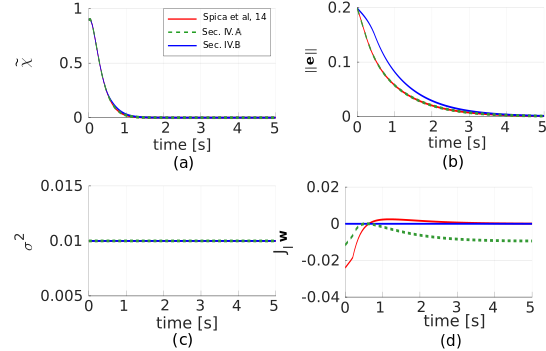}
    \caption{Comparison of the estimation strategies described in \cite{spica14} (\textit{Spica et al. 14}), Sec.~\ref{subsec:non_constant_depth} ($\dot{\chi}=0$ relaxed), and Sec.~\ref{subsec:constant_depth} ($\dot{\chi}=0$) . The initial inverse depth estimation error is $\tilde{\chi}=0.9~\text{m}^{-1}$ and the initial tracking error is $\|\mathbf{e}\|=0.2~\text{m}$. From top to bottom, it is shown the results of (a) the inverse depth estimation error, (b) the tracking error, (c) the persistence of excitation measurement $\sigma^2$, and (d) the constraint $J_l\mathbf{w}$ described in Theorem~\ref{the:theorem_generic}.}
    \label{fig:results_comparison_01}
\end{figure}

The behaviour of the depth estimation error is almost the same for the three methods - see Fig.~\ref{fig:results_comparison_01}(a). In fact, as shown in Fig.~\ref{fig:results_comparison_01}(c), the three strategies continuously fulfill the PE condition, given by $\sigma^2$, at its maximum value. Figure~\ref{fig:results_comparison_01}(b) shows that the feature tracking error converges slower for the method described in Sec.~\ref{subsec:constant_depth}. This is because the constraint $J_l\mathbf{w}=0$ imposes severe limitations on the the angular velocity vector. \textit{Spica et al. (2014)} guarantees asymptotic stability by driving the feature to the origin of the image frame, while the strategies proposed in this paper ensure that the constraints described in Theorem~\ref{the:theorem_generic} hold throughout the entire estimation process regardless of the feature coordinate. In particular, the constraint associated to $J_l\mathbf{w}$ can be seen in Fig.~\ref{fig:results_comparison_01}(d). For the method in Sec.~\ref{subsec:non_constant_depth}, $J_l\mathbf{w}$ is smaller or equal to zero. For the method in Sec.~\ref{subsec:constant_depth}, the constraint is always zero.

\begin{figure}
    \centering
    \includegraphics[width=.45\textwidth]{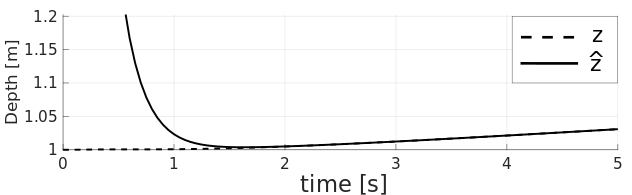}
    \caption{True depth ($z=1/\chi$) and its estimation ($\hat{z}=1/\hat{\chi}$) using the strategy described in Sec.~\ref{subsec:non_constant_depth} and the same setup as in Fig.~\ref{fig:results_comparison_01}}
    \label{fig:depth_change}
\end{figure}

For the same scenario, Fig.~\ref{fig:depth_change} shows the ground truth and the depth estimation using the method in Sec.~\ref{subsec:non_constant_depth}. In contrast to other continuous estimation strategies presented in the literature (namely \cite{spica14}), the method proposed in Sec.~\ref{subsec:non_constant_depth} ensures the depth estimation error converges to zero even thought the depth of the point with respect to the camera is not constant throughout the entire estimation process. 

In our formulation, the desired feature coordinate $\mathbf{s}_{des}$ can be time-varying. Figure~\ref{fig:s_des_time_varying} shows a scenario where the goal is to have the projection of the feature moving in a circular pattern. More specifically, we define $\mathbf{s}_{des} = 0.1[\cos(2\pi/10t), \sin(2\pi/10t)]^T$. As shown in Fig.~\ref{fig:s_des_time_varying}(a), the speed of convergence of the depth estimation error does not change when compared to the previous case (constant $\mathbf{s}_{des}$). Finally, Fig.~\ref{fig:s_des_time_varying}(b) and (c) show that while the depth estimation converges, the proposed control law is able to follow the time-varying signal $\mathbf{s}_{des}$.

\begin{figure}
    \centering
    \includegraphics[width=.49\textwidth]{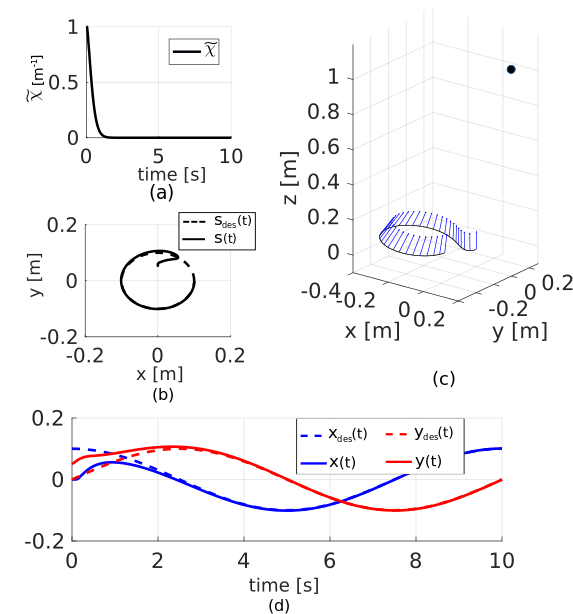}
    \caption{Assessing the performance of the proposed depth estimation framework when the desired feature coordinates ($\mathbf{s}_{des}(t)$) is time-varying. (a) shows the depth estimation error, (b) shows the desired and the current projection of the 3D point in the image plane, (c) illustrates the trajectory of the camera in a black line, the $z$--axis in a blue arrow, and the 3D point in black, and (d) the two previous signals over time per axis.}
    \label{fig:s_des_time_varying}
\end{figure}

%% file: files/conclusion.tex
\section{CONCLUSIONS}\label{sec:conclusions}
In this paper we analyze the required conditions for asymptotic stability of a class of depth estimation observers when the control inputs of the camera can be computed in an active manner. We applied the results for the depth estimation of a single 3D point. In contrast to previous works, our framework guarantees asymptotic stability when the feature coordinate does not converge to the origin of the image frame, nor its depth with respect to the camera is necessarily constant. We believe that relaxing the feature coordinates within the image frame while still providing asymptotic stability guarantees is paramount to apply incremental depth estimation in multiple point scenarios. Despite the relaxed constraints that allow a larger set of motions with theoretical guarantees for depth estimation, the numerical simulations shows that the proposed strategy performs similarly to related literature methods. In future work, we will extend our framework to multiple point and performs tests with a real robot/camera setup.